\documentclass[11pt,a4paper]{article}
\usepackage[hyperref]{eacl2021}
\usepackage{times}
\usepackage{latexsym}

\usepackage{amsmath,amsfonts,amssymb}
\usepackage{multirow}
\usepackage{algorithm}
\usepackage{algorithmic}

\usepackage{graphicx}
\usepackage{booktabs}
\usepackage{enumerate}

\usepackage{subfigure}
\usepackage{amsthm}
\usepackage{url}            
\usepackage{amsfonts}       
\usepackage{amsmath}
\newcommand{\xvec}{{\bf x}}

\newtheorem{lemma}{Lemma}
\newtheorem{theorem}{Theorem}
\newtheorem{corollary}{Corollary}

\usepackage{microtype}

\aclfinalcopy 
\def\aclpaperid{3} 

\title{Conditional Adversarial Networks for Multi-Domain Text Classification}

\author{Yuan Wu \textsuperscript{1}, Diana Inkpen \textsuperscript{2} and Ahmed El-Roby \textsuperscript{1} \\
  \textsuperscript{1}Carleton University, Ottawa, Ontario, Canada \\
  \textsuperscript{2}University of Ottawa, Ottawa, Ontario, Canada \\
  \texttt{\{yuan.wu3, Ahmed.ElRoby\}@carleton.ca}\quad \texttt{Diana.Inkpen@uottawa.ca} \\}

\date{}

\begin{document}
\maketitle
\begin{abstract}

In this paper, we propose conditional adversarial networks (CANs), a framework that explores the relationship between the shared features and the label predictions to impose more discriminability to the shared features, for multi-domain text classification (MDTC). The proposed CAN introduces a conditional domain discriminator to model the domain variance in both shared feature representations and class-aware information simultaneously and adopts entropy conditioning to guarantee the transferability of the shared features. We provide theoretical analysis for the CAN framework, showing that CAN's objective is equivalent to minimizing the total divergence among multiple joint distributions of shared features and label predictions. Therefore, CAN is a theoretically sound adversarial network that discriminates over multiple distributions. Evaluation results on two MDTC benchmarks show that CAN outperforms prior methods. Further experiments demonstrate that CAN has a good ability to generalize learned knowledge to unseen domains.

\end{abstract}


\section{Introduction}

Text classification is a fundamental task in Natural Language Processing (NLP) and has received constant attention due to its wide applications, ranging from spam detection to social media analytics \cite{pang2002thumbs,hu2004mining,choi2008learning,socher2012semantic,vo2015target}. Over the past couple of decades, supervised machine learning methods have shown dominant performance for text classification, such as naive bayes classifiers \cite{troussas2013sentiment}, support vector machines \cite{li2018new} and neural networks \cite{wu2020dualb}. In particular, with the advent of deep learning, neural network-based text classification models have gained impressive achievements. However, text classification is known to be highly domain-dependent, the same word could convey different sentiment polarities in different domains \cite{glorot2011domain}. For example, the word \textit{infantile} expresses neutral sentiment in baby product review (e.g., \textit{The infantile cart is easy to use}), while in book review, it indicates a negative polarity (e.g.,  \textit{This book is infantile and boring}). Thus a text classifier trained on one domain is likely to make spurious predictions on another domain whose distribution is different from the training data distribution. In addition, it is always difficult to collect sufficient labeled data for all interested domains. Therefore, it is of great significance to explore how to leverage available resources from related domains to improve the classification accuracy on the target domain.

The major line of approaches to tackle the above problem is multi-domain text classification (MDTC) \cite{li2008multi}, which can handle the scenario where labeled data exist for multiple domains, but in insufficient amounts to training an effective classifier for each domain. Deep learning models have yielded impressive performance in MDTC \cite{wu2020dual,wu2021mixup}. Most recent MDTC methods adopt the shared-private paradigm, which divides the latent space into two types: one is the shared feature space among domains with the aim of capturing domain-invariant knowledge, the other one is the private feature space for each domain which extracts domain-specific knowledge. To explicitly ensure the optimum separations among the shared latent space and multiple domain-specific feature spaces, the adversarial training \cite{goodfellow2014generative} is introduced in MDTC. By employing the adversarial training, the domain-specific features can be prevented from creeping into the shared latent space, which will lead to feature redundancy \cite{liu2017adversarial}. In adversarial training, a multinomial domain discriminator is trained against a shared feature extractor to minimize the divergences across different domains. When the domain discriminator and the shared feature extractor reach equilibrium, the learned shared features can be regarded as domain-invariant and used for the subsequent classification. The adversarial training-based MDTC approaches yield the state-of-the-art results \cite{liu2017adversarial,chen2018multinomial}. However, these methods still have a significant limitation: when the data distributions present complex structures, adversarial training may fail to perform global alignment among domains. Such a risk comes from the challenge that in adversarial training, only aligning the marginal distributions can not sufficiently guarantee the discriminability of the learned features. The features with different labels may be aligned, as shown in Figure \ref{Fig8}. The fatal mismatch can lead to weak discriminability of the learned features.

\begin{figure}
    \centering
    \includegraphics[width=0.9\columnwidth]{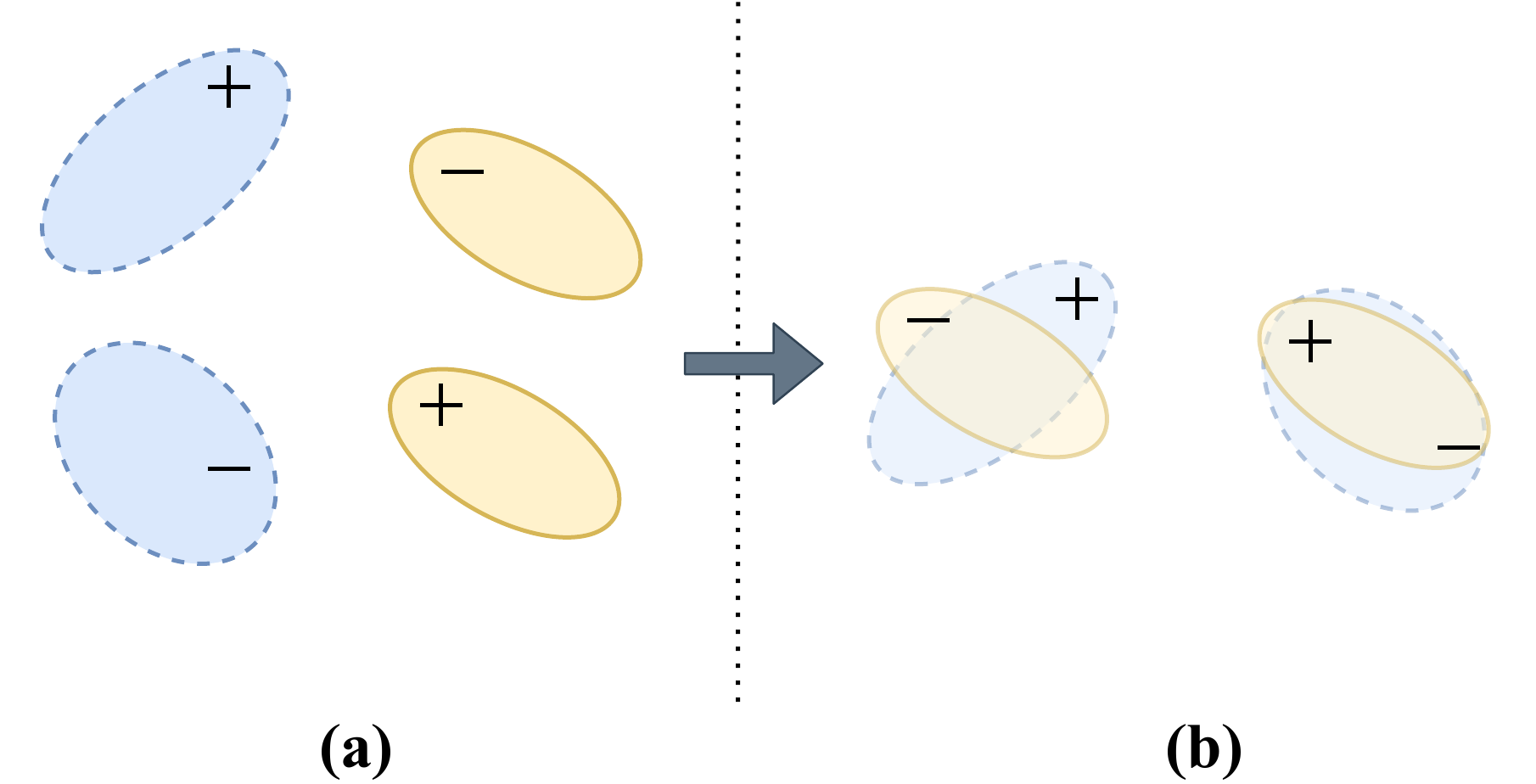}
    \caption{The mismatch risk when aligning the marginal distributions in MDTC, we present the case containing two domains $D_1$ and $D_2$. The blue regions denote distributions of $D_1$, and the yellow regions denote distributions of $D_2$. (a) The scenario before performing domain alignment. (2) When aligning the marginal distributions, a mismatch may occur with regard to the label.
	}
    \label{Fig8}
\end{figure}

In this paper, motivated by the conditional generative adversarial networks (CGANs), which aligns distributions of real and generated images via conditioning the generator and discriminator on extra information \cite{mirza2014conditional}, we propose conditional adversarial networks (CANs) to address the aforementioned challenge. The CAN method introduces a conditional domain discriminator that models domain variance in both shared features and label predictions, exploring the relationship between shared feature representations and class-aware information conveyed by label predictions to encourage the shared feature extractor to capture more discriminative information. Moreover, we use entropy conditioning to avoid the risk of conditioning on the class-aware information with low certainty. The entropy conditioning strategy can give higher priority to easy-to-transfer instances. We also provide a theoretical analysis demonstrating the validity of CANs. Our approach adopts the shared-private paradigm. We validate the effectiveness of CAN on two MDTC benchmarks. It can be noted that CAN outperforms the state-of-the-art methods for both datasets. Finally, we empirically illustrate that CAN has the ability to generalize in cases where no labeled data exist for a subset of domains. The contributions of our work are listed as follows:

\begin{itemize}
    \item We propose conditional adversarial networks (CANs) for multi-domain text classification which incorporate conditional domain discriminator and entropy conditioning to perform alignment on the joint distributions of shared features and label predictions to improve the system performance.
    \item We present the theoretical analysis of the CAN framework, demonstrating that CANs are minimizers of divergences among multiple joint distributions of shared features and label predictions, and providing the condition where the conditional domain discriminator reaches its optimum.
    \item We evaluate the effectiveness of CAN on two MDTC benchmarks. The experimental results show that CAN yields state-of-the-art results. Moreover, further experiments on unsupervised multi-source domain adaptation demonstrate that CAN has a good capacity to generalize to unseen domains.
\end{itemize}


\section{Related Work}

Multi-domain text classification (MDTC) was first proposed by \cite{li2008multi}, aiming to simultaneously leverage all existing resources across different domains to improve the system performance. Currently, there are two main streams for MDTC: one strand exploits covariance matrix to model the relationship across domains \cite{dredze2008online,saha2011online,zhang2012convex}; the other strand is based on neural networks, sharing the first several layers for each domain to extract low-level features and generating outputs with domain-specific parameters. The multi-task convolutional neural network (MT-CNN) utilizes a convolutional layer in which only the lookup table is shared for better word embeddings \cite{collobert2008unified}. The collaborative multi-domain sentiment classification (CMSC) combines a classifier that learns common knowledge among domains with a set of classifiers, one per domain, each of which captures domain-dependent features to make the final predictions \cite{wu2015collaborative}. The multi-task deep neural network (MT-DNN) maps arbitrary text queries and documents into semantic vector representations in a low dimensional latent space and combines tasks as disparate as operations necessary for classification \cite{liu2015representation}.

Pioneered by the generative adversarial network (GAN) \cite{goodfellow2014generative}, adversarial learning has been firstly proposed for image generation. \cite{ganin2016domain} applies adversarial learning in domain adaptation to extract domain-invariant features across two different distributions (binary adversarial learning). \cite{zhao2017multiple} extends it to multiple adversarial learning, enabling the model to learn domain-invariant representations across multiple domains. However, only considering domain-invariant features can not provide optimal solutions for MDTC, because domain-specific information also plays an important role in training an effective classifier. \cite{bousmalis2016domain} proposes the shared-private paradigm to combine domain-invariant features with domain-specific ones to perform classification, illustrating that this scheme can improve system performance. To date, many state-of-the-art MDTC models adopt adversarial learning and shared-private paradigm. The adversarial multi-task learning for text classification (ASP-MTL) utilizes long short-term memory (LSTM) without attention as feature extractors and introduces orthogonality constraints to encourage the shared and private feature extractors to encode different aspects of the inputs \cite{liu2017adversarial}. The multinomial adversarial network (MAN) exploits two forms of loss functions to train the domain discriminator: the least square loss (MAN-L2) and negative log-likelihood loss (MAN-NLL) \cite{chen2018multinomial}. The multi-task learning with bidirectional language models for text classification (MT-BL) introduces language modeling as an auxiliary task to encourage the domain-specific feature extractors to capture more syntactic and semantic information, and a uniform label distribution-based loss constraint to the shared feature extractor to enhance the ability to learn domain-invariant features \cite{yang2019multi}.

Adversarial learning has several advantages, such as Markov chains are not needed and no inference is required during learning \cite{mirza2014conditional}. However, there still exists an issue in adversarial learning. When data distributions embody complex structures, adversarial learning can fail in performing the global alignment. The conditional generative adversarial network (CGAN) is proposed to address this problem \cite{mirza2014conditional}. In CGAN, both the generator and discriminator are conditioned on some extra information, such as labels or data from other modalities, to yield better results. Sharing some spirit of CGAN, this paper extends conditional adversarial learning in MDTC, enabling a domain discriminator on the shared features while conditioning it on the class-aware information conveyed by the label predictions. Moreover, in order to guarantee the generalizability of the learned features, we also utilize the entropy conditioning strategy.


\section{Approach}


\begin{figure*}
    \centering
    \includegraphics[width=1.3\columnwidth]{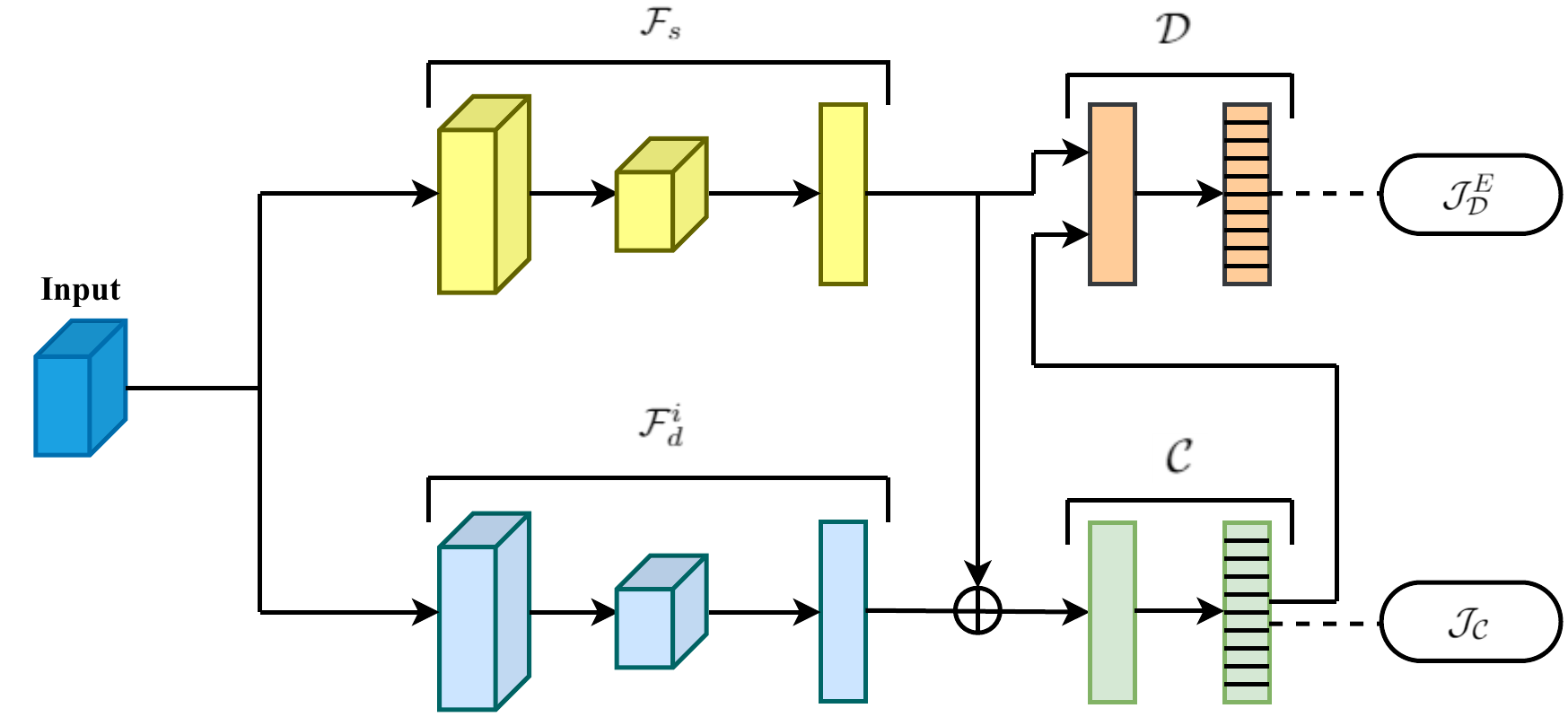}
    \caption{The architecture of the CAN model. A shared feature extractor $\mathcal{F}_s$ learns to capture domain-invariant features;  
	each domain-specific feature extractor $\mathcal{F}^i_{d}$
	learns to capture domain-dependent features;  
	a conditional domain discriminator $\mathcal{D}$ models shared feature distributions by conditioning on discriminative information provided by label predictions; a classifier $\mathcal{C}$ is used to conduct text classification; $\mathcal{J}_\mathcal{C}$ is the classification loss function; and $\mathcal{J}_{\mathcal{D}}^E$ is the entropy conditioning adversarial loss function which guides the domain-invariant feature extraction.
	}
    \label{Fig2}
\end{figure*}


In this paper, we consider MDTC tasks in the following setting. Assume there exist $M$ domains $\{D_i\}_{i=1}^{M}$. For each domain, both labeled and unlabeled samples are taken into consideration. Specifically, $D_i$ contains two parts: a limited amount of labeled samples $\mathbb{L}_i=\{(\xvec_j,y_j)\}_{j=1}^{l_i}$; and a large amount of unlabeled samples $\mathbb{U}_i=\{\xvec_j\}_{j=1}^{u_i}$. The challenge of MDTC lies in how to improve the system performance of mapping the input $\xvec$ to its corresponding label $y$ by leveraging all available resources across different domains. The performance is measured as the average classification accuracy across the $M$ domains. 

\subsection{Model Architecture}

We propose conditional adversarial networks (CANs), as shown in Figure \ref{Fig2}, which adopt the shared-private scheme and consist of four components: a shared feature extractor $\mathcal{F}_s$, a set of domain-specific feature extractors $\{\mathcal{F}_d^i\}_{i=1}^M$, a conditional domain discriminator $\mathcal{D}$, and a text classifier $\mathcal{C}$. The shared feature extractor $\mathcal{F}_s$ learns to capture domain-invariant features that are beneficial to classification across all domains, while each domain-specific feature extractor $\mathcal{F}_d^i$ aims to learn knowledge that is unique to its own domain. The architecture of these feature extractors are flexible and can be decided based on the practical task. For instance, it can adopt the form of a convolutional neural network (CNN), a recurrent neural network (RNN), or a multiple layer perceptron (MLP). Here, a feature extractor generates vectors with a fixed length, which is considered as the hidden representation of certain input. The classifier $\mathcal{C}$ takes the concatenation of a shared feature and a domain-specific feature as its input and outputs label probabilities. The conditional domain discriminator $\mathcal{D}$ takes the concatenation of a shared feature and the prediction of the given instance provided by $\mathcal{C}$ as its input and predicts the likelihood of that instance coming from each domain.

\subsection{Conditional Adversarial Training}

Adversarial learning has been successfully investigated in minimizing divergences among domains \cite{chen2018multinomial,zhao2017multiple}. In standard adversarial learning for MDTC, a two-player mini-max game is conducted between a domain discriminator and a shared feature extractor: the domain discriminator is trained to distinguish features across different domains, and the shared feature extractor aims to deceive the discriminator. By performing this mini-max optimization, the domain-invariant features can be learned. The error function of the domain discriminator corresponds well to the divergences among domains. Most MDTC methods align the marginal distributions. However, the transferability with representations transition from general to specific along deep networks is decreasing significantly \cite{yosinski2014transferable}, only adapting the marginal distributions is not sufficient to guarantee the global alignment. In addition, when the data distributions embody complex structures, which is a real scenario for NLP applications, there is a high risk of failure by matching features with different labels.

Recent advances in the conditional generative adversarial network (CGAN) disclose that better alignment on two different distributions can be obtained by conditioning the generator and discriminator on class-aware information \cite{mirza2014conditional}. The discriminative information provided by the label prediction potentially reveals the structure information underlying the data distribution. Thus, conditional adversarial learning can better model the divergences among domains on shared feature representations and label predictions. Unlike the prior works that adapting the marginal distributions \cite{liu2017adversarial,chen2018multinomial}, our proposed CAN framework is formalized on aligning joint distributions of shared features and label predictions. There exist two training flows in our model. Due to the nature of adversarial learning, the conditional domain discriminator is updated with a separate optimizer, while the other components of CAN are trained with the main optimizer. These two flows are supposed to complement each other. Denote $\mathcal{L}_{\mathcal{C}}$ and $\mathcal{L}_{\mathcal{D}}$ as the loss functions of the classifier $\mathcal{C}$ and the conditional domain discriminator $\mathcal{D}$, respectively. We utilize the negative log-likelihood (NLL) loss to encode these two loss functions:

\begin{align}
    \mathcal{L}_{\mathcal{C}}(\widetilde{y}, y) = -\log P(\widetilde{y} = y)
\end{align}

\begin{align}
    \mathcal{L}_{\mathcal{D}}(\widetilde{d}, d) = -\log P(\widetilde{d} = d)
\end{align}

\noindent where $y$ is the true label, $\widetilde{y}$ is the label prediction, $d$ is the domain index and $\widetilde{d}$ is the domain prediction. Therefore, we formulate CAN as a mini-max optimization problem with two competitive terms defined as follows: 

\begin{align}
    \mathcal{J}_{\mathcal{C}} = \sum_{i=1}^M \mathbb{E}_{(\xvec,y)\sim\mathbb{L}_i}[\mathcal{L}_{\mathcal{C}}(\mathcal{C}_i,y)]
\end{align}

\begin{equation}
    \begin{aligned}
    \mathcal{J}_{\mathcal{D}} = \sum_{i=1}^M \mathbb{E}_{\xvec \sim \mathbb{L}_i \cup \mathbb{U}_i}[\mathcal{L}_{\mathcal{D}}(\mathcal{D}([\mathcal{F}_s(\xvec), \mathcal{C}_i]),d)]
    \end{aligned}
\end{equation}

\noindent where $[\cdot,\cdot]$ is the concatenation of two vectors, $\mathcal{C}_i=\mathcal{C}([\mathcal{F}_s(\xvec),\mathcal{F}_d^i(\xvec)])$ is the prediction probability of the given instance $\xvec$. $\mathcal{C}$ and $\mathcal{D}$ adopt MLPs with a softmax layer on top. For the domain-specific feature extractors $\{\mathcal{F}_d^i\}_{i=1}^M$, the training is straightforward, as their objective is simple: help $\mathcal{C}$ perform better classification. While the shared feature extractor $\mathcal{F}_s$ has two goals: (1) help $\mathcal{C}$ reduce prediction errors, and (2) confuse $\mathcal{D}$ to reach equilibrium.

\subsection{Entropy Conditioning}

We condition the domain discriminator $\mathcal{D}$ on the joint variable $(f,c) = (\mathcal{F}_s(\xvec), \mathcal{C}_i)$. For brevity, here we use $f$ and $c$ to denote $\mathcal{F}_s(\xvec)$ and $\mathcal{C}_i$, respectively. If we enforce different instances to have equal importance, the hard-to-transfer instances with uncertain predictions may deteriorate the system performance \cite{saito2019semi}. In order to alleviate the harmful effects introduced by the hard-to-transfer instances, we introduce the entropy criterion $E(c)=-\sum_{k=1}^2[c_k log c_k]$ to quantify the uncertainty of label predictions, where $c_k$ is the probability of predicting an instance to category $k$ (negative: k = 1, positive: k = 2). By using the entropy conditioning, the easy-to-transfer instances with certain predictions are given higher priority. We reweigh these instances by an entropy-aware term: $w(c) = 1 + e^{-E(c)}$. Therefore, the improved $\mathcal{J}_{\mathcal{D}}$ is defined as:

\begin{equation}
    \begin{aligned}
    \mathcal{J}_{\mathcal{D}}^E &= \sum_{i=1}^M \mathbb{E}_{\xvec \sim \mathbb{L}_i \cup \mathbb{U}_i}[w(c) \mathcal{L}_{\mathcal{D}}(\mathcal{D}([\mathcal{F}_s(\xvec), \mathcal{C}_i]),d)]
    \end{aligned}
\end{equation}

\noindent Therefore, the mini-max game of CAN is formulated as:

\begin{align}
    \min_{\mathcal{F}_s,\{\mathcal{F}_d^i\}_{i=1}^M, C}\max_{\mathcal{D}} \quad \mathcal{J}_{\mathcal{C}} + \lambda \mathcal{J}_{\mathcal{D}}^E
\end{align}

\noindent where $\lambda$ is a hyperparameter balancing the two objectives. The entropy conditioning empowers the entropy minimization principle \cite{grandvalet2005semi} and controls the certainty of the predictions, enabling CAN have the ability to generalize on unseen domains with no labeled data. The CAN training is illustrated in Algorithm  \ref{trainingalg}.

\begin{algorithm}
\caption{Stochastic gradient descent training algorithm}\label{trainingalg}
\begin{algorithmic}[1]
\STATE{\bf Input:} labeled data $\mathbb{L}_i$ and unlabeled data $\mathbb{U}_i$ in $M$ domains; 
	a hyperparameter $\lambda$.
\FOR{number of training iterations}
	\STATE Sample labeled mini-batches from the multiple domains $B^\ell=\{B^\ell_1,\cdots, B^\ell_M\}$.
	\STATE Sample unlabeled mini-batches from the multiple domains $B^u=\{B^u_1,\cdots, B^u_M\}$.
	\STATE Calculate $loss = \mathcal{J}_{\mathcal{C}}+\lambda \mathcal{J}_{\mathcal{D}}^E$ on $B^\ell$ and $B^u$;\\
	Update $\mathcal{F}_s$, $\{\mathcal{F}_d^i\}_{i=1}^M$, $\mathcal{C}$ by descending along the gradients $\Delta loss$.\\[.2ex] 
	\STATE Calculate $l_D=\mathcal{J}_{\mathcal{D}}^E$ on $B^\ell$ and $B^u$;\\
	Update $\mathcal{D}$ by ascending along the gradients $\Delta l_D$.\\[.2ex]
\ENDFOR
\end{algorithmic}
\end{algorithm}

\subsection{Theoretical Analysis}

In this section, we present an analysis showing the validity of the CAN approach for MDTC. All proofs are given in the Appendix. The objective of CAN is equivalent to minimizing the total divergence among the $M$ joint distributions. First, we define different joint distributions as $P_i(f,c) \triangleq P(f = \mathcal{F}_s(\xvec), c = \mathcal{C}_i | \xvec \in D_i)$. Combining $\mathcal{L}_\mathcal{D}$ with $\mathcal{J}_\mathcal{D}$, the objective of $\mathcal{D}$ can be written as:

\begin{align}
    \mathcal{J}_\mathcal{D} = -\sum_{i=1}^M \mathbb{E}_{(f,c)\sim P_i}[\log\mathcal{D}_i([f,c])]
\end{align}

\noindent where $\mathcal{D}_i([f,c])$ yields the probability of the vector $([f,c])$ coming from the $i$-th domain. We first derive that CAN could achieve its optimum if and only if all $M$ joint distributions are identical.

\begin{lemma}
\label{lem1}
 For any given $\mathcal{F}_s$, $\{\mathcal{F}_d^i\}_{i=1}^M$ and $\mathcal{C}$, the optimum conditional domain discriminator $\mathcal{D}^*$ is:
\begin{align}
    \mathcal{D}_i^*([f,c]) = \frac{P_i(f,c)}{\sum_{j=1}^M P_j(f,c)}
\end{align}
\end{lemma}

\noindent Then we provide the main theorem for the CAN framework:

\begin{theorem}
\label{theo1}
Let $\widetilde{P}(f,c) = \frac{\sum_{i=1}^M P_i(f,c)}{M}$, when $\mathcal{D}$ is trained to its optimum $\mathcal{D}^*$, we have:
\begin{align}
    \mathcal{J}_{\mathcal{D}^*}=M \log M-\sum_{i=1}^M KL(P_i(f,c)||\widetilde{P}(f,c))
\end{align}
where $KL(\cdot)$ is the Kullback-Leibler (KL) divergence \cite{aslam2007query} of each joint distribution $P_i(f,c)$ to the centroid $\widetilde{P}(f,c)$.
\end{theorem}

\noindent Finally, considering the non-negativity and convexity of the 
KL-divergence \cite{brillouin2013science}, we have:

\begin{corollary}
\label{coro1}
When $\mathcal{D}$ is trained to its optimum $\mathcal{D}^*$, $\mathcal{J}_{\mathcal{D}^*}$ is $M \log M$. The optimum can be obtained if and only if $P_1(f,c) = P_2(f,c) = ... = P_M(f,c) = \widetilde{P}(f,c)$.
\end{corollary}

\noindent Therefore, by using conditional adversarial training, we can train the conditional domain discriminator on the joint variable $(f,c)$ to minimize the total divergence across different domains, yielding promising performance in MDTC tasks.


\section{Experiments}

We evaluate the effectiveness of the CAN model on both MDTC and unsupervised multi-source domain adaptation tasks. The former refers to the setting where the test data falls into one of the $M$ domains, and the latter refers to the setting where the test data comes from an unseen domain without labels. Moreover, an ablation study is provided for further analysis of the CAN model.

\subsection{Experimental Settings}
\paragraph{Dataset}


We conduct experiments on two MDTC benchmarks: the Amazon review dataset \cite{blitzer2007biographies} and the FDU-MTL dataset \cite{liu2017adversarial}. The Amazon review dataset consists of four domains: books, DVDs, electronics, and kitchen. For each domain, there exist 2,000 instances: 1,000 positive ones and 1,000 negative ones. All data was pre-processed into a bag of features (unigrams and bigrams), losing all word order information. In our experiments, the 5,000 most frequent features are used, representing each review as a 5,000-dimensional vector. The FDU-MTL dataset is a more complicated dataset, which contains 16 domains: books, electronics, DVDs, kitchen, apparel, camera, health, music, toys, video, baby, magazine, software, sport, IMDB, and MR. All data in the FDU-MTL dataset are raw text data, tokenized by the Stanford tokenizer. The detailed statistics of the FDU-MTL dataset are listed in the Appendix.


\paragraph{Implementation Details}

All experiments are implemented by using \textbf{PyTorch}. The CAN has one hyperparameter: $\lambda$, which is fixed as 1 in all experiments, the parameter sensitivity analysis is presented in the Appendix. We use Adam optimizer \cite{kingma2014adam}, with the learning rate 0.0001, for training. The batch size is 8. We adopt the same model architecture as in \cite{chen2018multinomial}. For the Amazon Review dataset, MLPs are used as feature extractors, with an input size of 5,000. Each feature extractor is composed of two hidden layers, with size 1,000 and 500, respectively. The output size of the shared feature extractor is 128 while 64 for the domain-specific ones. The dropout rate is 0.4 for each component. Classifier and discriminator are MLPs with one hidden layer of the same size as their input ($128+64$ for classifier and $128+2$ for discriminator). ReLU is used as the activation function. For the FDU-MTL dataset, CNN with a single convolutional layer is used as the feature extractor. It uses different kernel sizes $(3,4,5)$, and the number of kernels is 200. The input of the convolutional layer is a 100-dimensional vector, obtained by using word2vec \cite{mikolov2013efficient}, for each word in the input sequence.

\subsection{Multi-Domain Text Classification}

\begin{table*}[t]
\caption{\label{font-table} MDTC classification accuracies on the Amazon review dataset.}\smallskip
\label{table_ref1}
\centering
\resizebox{1.60\columnwidth}{!}{
\smallskip\begin{tabular}{ l|  c c c c c c c}
\hline
	Domain & CMSC-LS & CMSC-SVM & CMSC-Log & MAN-L2 & MAN-NLL & CAN(Proposed)\\
\hline
Books &  82.10 & 82.26 & 81.81 & 82.46 & 82.98 & $\mathbf{83.76\pm0.20}$ \\
DVD &  82.40 & 83.48 & 83.73 & 83.98 & 84.03 & $\mathbf{84.68\pm0.16}$ \\
Electr.  & 86.12 & 86.76 & 86.67 & 87.22 & 87.06 & $\mathbf{88.34\pm0.14}$ \\
Kit.  &  87.56 & 88.20 & 88.23 & 88.53 & 88.57 & $\mathbf{90.03\pm0.19}$\\
\hline
AVG  &  84.55 & 85.18 & 85.11 & 85.55 & 85.66 & $\mathbf{86.70\pm0.11}$\\
\hline
\end{tabular}}
\end{table*}
\begin{table*}[t]
\caption{\label{font-table} MDTC classification accuracies on the FDU-MTL dataset. }\smallskip
\label{table_ref2}
\centering
\resizebox{1.60\columnwidth}{!}{
\begin{tabular}{ l| c c c c c c c}
\hline
	Domain & MT-CNN & MT-DNN & ASP-MTL & MAN-L2 & MAN-NLL & MT-BL & CAN(Proposed)\\
\hline
books & 84.5 & 82.2 & 84.0 & 87.6 & 86.8 & $\mathbf{89.0}$ & $87.8\pm0.2$ \\
electronics & 83.2 & 81.7 & 86.8 & 87.4 & 88.8 & 90.2 & $\mathbf{91.6\pm0.5}$ \\
dvd & 84.0 & 84.2 & 85.5 & 88.1 & 88.6 & 88.0 & $\mathbf{89.5\pm0.4}$ \\
kitchen & 83.2 & 80.7 & 86.2 & 89.8 & 89.9 & 90.5 & $\mathbf{90.8\pm0.3}$\\
apparel & 83.7 & 85.0 & 87.0 & $\mathbf{87.6}$ & $\mathbf{87.6}$ & 87.2 & $87.0\pm0.7$\\
camera & 86.0 & 86.2 & 89.2 & 91.4 & 90.7 & 89.5 & $\mathbf{93.5\pm0.1}$\\
health & 87.2 & 85.7 & 88.2 & 89.8 & 89.4 & $\mathbf{92.5}$ & $90.4\pm0.6$ \\
music & 83.7 & 84.7 & 82.5 & 85.9 & 85.5 & 86.0 & $\mathbf{86.9\pm0.1}$ \\
toys & 89.2 & 87.7 & 88.0 & 90.0 & 90.4 & $\mathbf{92.0}$ & $90.0\pm0.3$ \\
video & 81.5 & 85.0 & 84.5 & 89.5 & $\mathbf{89.6}$ & 88.0 & $88.8\pm0.4$ \\
baby & 87.7 & 88.0 & 88.2 & 90.0 & 90.2 & 88.7 & $\mathbf{92.0\pm0.2}$ \\
magazine & 87.7 & 89.5 & 92.2 & 92.5 & 92.9 & 92.5 & $\mathbf{94.5\pm0.5}$ \\ 
software & 86.5 & 85.7 & 87.2 & 90.4 & 90.9 & $\mathbf{91.7}$ & $90.9\pm0.2$ \\
sports & 84.0 & 83.2 & 85.7 & 89.0 & 89.0 & 89.5 & $\mathbf{91.2\pm0.7}$ \\
IMDb & 86.2 & 83.2 & 85.5 & 86.6 & 87.0 & 88.0 & $\mathbf{88.5\pm0.6}$\\
MR & 74.5 & 75.5 & 76.7 & 76.1 & 76.7 & 75.7 & $\mathbf{77.1\pm0.9}$\\
\hline
AVG & 84.5 & 84.3 & 86.1 & 88.2 & 88.4  & 88.6  & $\mathbf{89.4\pm0.1}$ \\
\hline
\end{tabular} }
\end{table*}

\paragraph{Comparison Methods}

We first conduct experiments of multi-domain text classification. The CAN model is compared with a number of state-of-the-art methods, which are listed below:

\begin{itemize}
    \item MT-CNN: A CNN-based model which shares the lookup table across domains for better word embeddings \cite{collobert2008unified}.
    \item MT-DNN: The multi-task deep neural network model with bag-of-words input and MLPs, in which a hidden layer is shared \cite{liu2015representation}.
    \item CMSC-LS, CMSC-SVM, CMSC-Log: The collaborative multi-domain sentiment classification method combines an overall classifier across domains and a set of domain-dependent classifiers to make the final prediction. The models are trained on least square loss, hinge loss, and log loss, respectively \cite{wu2015collaborative}.
    \item ASP-MTL: The adversarial multi-task learning framework of text classification, which adopts the share-private scheme, adversarial learning, and orthogonality constraints \cite{liu2017adversarial}.
    \item MAN-L2, MAN-NLL: The multinomial adversarial network for multi-domain text classification \cite{chen2018multinomial}. This model uses two forms of loss functions to train domain discriminator: least square loss and negative log-likelihood loss.
    \item MT-BL: The multi-task learning with bidirectional language models for text classification, which adds language modeling and a uniform label distribution-based loss constraint to the domain-specific feature extractors and shared feature extractor, respectively \cite{yang2019multi}.
\end{itemize}

All the comparison methods use the standard partitions of the datasets. Thus, we cite the results from \cite{chen2018multinomial,yang2019multi} for fair comparisons.

\paragraph{Results}

\begin{table}
\caption{\label{font-table} Ablation study analysis.}\smallskip
\label{table_ref4}
\centering
\resizebox{0.8\columnwidth}{!}{
\begin{tabular}{ l| c c c c c }
\hline
Method & Books & DVD & Electr. & Kit. & AVG\\
\hline
	CAN (full)& 83.76 & 84.68 & 88.34 & 90.03 & 86.70 \\
CAN w/o C & 82.45 & 84.45 & 87.30 & 89.65 & 85.96 \\
CAN w/o E & 83.60 & 84.80 & 87.70 & 89.40 & 86.38 \\
CAN w/o CE & 82.98 & 84.03 & 87.06 & 88.57 & 85.66 \\
\hline
\end{tabular}}
\end{table}

\begin{table*}
\caption{\label{font-table} Unsupervised multi-source domain adaptation results on the Amazon review dataset.}\smallskip
\label{table_ref3}
\centering
\resizebox{1.60\columnwidth}{!}{
\begin{tabular}{ l| c c c c c c c c}
\hline
Domain & MLP & mSDA & DANN & MDAN(H) & MDAN(S) & MAN-L2 & MAN-NLL & CAN(Proposed)\\
\hline
Books & 76.55 & 76.98 & 77.89 & 78.45 & 78.63 & 78.45 & 77.78 & $\mathbf{78.91}$\\
DVD & 75.88 & 78.61 & 78.86 & 77.97 & 80.65 & 81.57 & 82.74 & $\mathbf{83.37}$\\
Elec. & 84.60 & 81.98 & 84.91 & 84.83 & $\mathbf{85.34}$ & 83.37 & 83.75 & 84.76\\
Kit. & 85.45 & 84.26 & 86.39 & 85.80 & 86.26 & 85.57 & 86.41 & $\mathbf{86.75}$\\
\hline
AVG & 80.46 & 80.46 & 82.01 & 81.76 & 82.72 & 82.24 & 82.67 & $\mathbf{83.45}$\\
\hline
\end{tabular}}
\end{table*}


We conduct MDTC experiments following the setting of \cite{chen2018multinomial}: A 5-fold cross-validation is implemented for the Amazon review dataset. All data is divided into 5 folds per domain: three of the five folds are used as the training set, one is the validation set, and the remaining one is treated as the test set. The 5-fold average test accuracy is reported. All reports are based on 5 runs.

Table \ref{table_ref1} and Table \ref{table_ref2} show the experimental results on the Amazon review dataset and the FDU-MTL dataset, respectively. From Table \ref{table_ref1}, it can be seen that our model yields state-of-the-art results not only for the average classification accuracy, but also on each individual domain.

From the experimental results on the FDU-MTL dataset, reported in Table \ref{table_ref2}, we can see that the CAN model obtains the best accuracies on 10 of 16 domains and achieves the best result in terms of the average classification accuracy. The experimental results on these two MDTC benchmarks illustrate the efficacy of our model. 


\paragraph{Ablation Study}

The CAN model adopts a conditional domain discriminator and entropy conditioning. In this section, we investigate how these two strategies impact the performance of our model on the Amazon review dataset. In particular, three ablation variants are evaluated: (1) CAN w/o C, the variant of the proposed CAN model without conditioning the domain discriminator on label predictions, which utilizes the standard domain discriminator and entropy conditioning; (2) CAN w/o E, the variant of the proposed CAN model without the entropy conditioning, which hence imposes equal importance to different instances; (3) CAN w/o CE, the variant of the proposed CAN model which only uses standard adversarial training for domain alignment. The results of the ablation study are shown in Table \ref{table_ref4}, where we can see that all variants produce inferior results, and the full model gives the best performance. Thus, it indicates that both strategies contribute to the CAN model. 

\subsection{Unsupervised Multi-Source Domain Adaptation}

In the MDTC scenario, the model requires labeled training data from each domain. However, in reality, many domains may have no labeled data at all. Therefore, it is important to evaluate the performance of unseen domains for MDTC models. 

In the unsupervised multi-source domain adaptation setting, we have multiple source domains with both labeled and unlabeled data and one target domain with only unlabeled data. The CAN has the ability to learn domain-invariant representations on unlabeled data, and thus it can be generalized to unseen domains. Since the target domain has no labeled data at all, the domain discriminator is updated only on unlabeled data in this setting. When conducting text classification on the target domain, we only feed the shared feature to $\mathcal{C}$ and set the domain-specific feature vector to 0.

We conduct the experiments on the Amazon review dataset. In the experiments, three of the four domains are regarded as the source domains, and the remaining one is used as the target one. The evaluations are conducted on the target domain. In order to validate CAN's effectiveness, we compare CAN with several domain-agnostic methods, including: (1) the MLP model; (2) the marginalized denoising autoencoder (mSDA) \cite{chen2012marginalized}; (3) the domain adversarial neural network (DANN) \cite{ganin2016domain}, these methods ignore the differences among domains. And certain state-of-the-art unsupervised multi-source domain adaptation methods: (4) the multi-source domain adaptation neural network (MDAN(H) and MDAN(S)) \cite{zhao2017multiple}; (5) the multinomial adversarial network (MAN-L2 and MAN-NLL) \cite{chen2018multinomial}. When training the domain-agnostic methods, the data in the multiple source domains are combined together as a single source domain.   

In Table \ref{table_ref3}, we observe that the CAN model outperforms all the comparison methods on three out of four domains. In terms of the average classification accuracy, the CAN method achieves superior performance. This suggests that our model has a good ability to generalize on unseen domains.

\section{Conclusion}
In this paper, we propose conditional adversarial networks (CANs) for multi-domain text classification. This approach can perform alignment on joint distributions of shared features and label predictions to improve the system performance. The CAN approach adopts the shared-private paradigm, trains domain discriminator by conditioning it on discriminative information conveyed by the label predictions to encourage the shared feature extractor to capture more discriminative information, and exploits entropy conditioning to guarantee the transferability of the learned features. Experimental results on two MDTC benchmarks demonstrate that the CAN model can not only improve the system performance on MDTC tasks effectively but also boost the generalization ability when tackling unseen domains.


\bibliography{anthology,eacl2021}
\bibliographystyle{acl_natbib}

\appendix

\section{Appendix}
\label{sec:appendix}

\subsection{Proofs for CAN}

Assume there exist $M$ domains, for each domain $\mathcal{D}_i$, we have a joint distribution defined as:

\begin{equation}
    \begin{aligned}
    P_i(f,c) &\triangleq P(f=\mathcal{F}_s(\xvec), c=\mathcal{C}_i|\xvec\in\mathcal{D}_i)
    \end{aligned}
\end{equation}

\begin{table}
\caption{\label{font-table} Statistics of the FDU-MTL dataset}\smallskip
\label{table_ref8}
\centering
\resizebox{1.00\columnwidth}{!}{
\begin{tabular}{ l| c c c c c c c c c}
\hline
Domain & Train & Dev. & Test & Unlabeled & Avg. L & Vocab.\\
\hline
Books & 1400 & 200 & 400 & 2000 & 159 & 62K \\
Electronics & 1398 & 200 & 400 & 2000 & 101 & 30K \\
DVD & 1400 & 200 & 400 & 2000 & 173 & 69K \\
Kitchen & 1400 & 200 & 400 & 2000 & 89 & 28K \\
Apparel & 1400 & 200 & 400 & 2000 & 57 & 21K \\
Camera & 1397 & 200 & 400 & 2000 & 130 & 26K \\
Health & 1400 & 200 & 400 & 2000 & 81 & 26K \\
Music & 1400 & 200 & 400 & 2000 & 136 & 60K \\
Toys & 1400 & 200 & 400 & 2000 & 90 & 28K \\
Video & 1400 & 200 & 400 & 2000 & 156 & 57K \\
Baby & 1300 & 200 & 400 & 2000 & 104 & 26K \\
Magazine & 1370 & 200 & 400 & 2000 & 117 & 30K \\
Software & 1315 & 200 & 400 & 475 & 129 & 26K \\
Sports & 1400 & 200 & 400 & 2000 & 94 & 30K \\
IMDB & 1400 & 200 & 400 & 2000 & 269 & 44K \\
MR & 1400 & 200 & 400 & 2000 & 21 & 12K \\
\hline
\end{tabular}}
\end{table}

\noindent where $\mathcal{C}_i=\mathcal{C}([\mathcal{F}_s(\xvec),\mathcal{F}_d^i(\xvec)])$ is the  prediction probability of the given instance $\xvec$, $[\cdot,\cdot]$ is the concatenation of two vectors. The objective of $\mathcal{D}$ is to minimize $\mathcal{J}_{\mathcal{D}}$:

\begin{align}
    \mathcal{J}_\mathcal{D} = -\sum_{i=1}^M \mathbb{E}_{(f,c)\sim P_i}[\log\mathcal{D}_i([f,c])]
\end{align}

\noindent where $\mathcal{D}_i([f,c])$ is the probability of the vector $([f,c])$ coming from the $i$-th domain. Therefore, we have:

\begin{align}
\label{eq3}
    \sum_{i=1}^M \mathcal{D}_i([f,c]) = 1
\end{align}

\begin{lemma}
\label{lem1}
For any given $\mathcal{F}_s$, $\{\mathcal{F}_d^i\}_{i=1}^M$ and $\mathcal{C}$, the optimum conditional domain discriminator $\mathcal{D}^*$ is:
\begin{align}
    \mathcal{D}_i^*([f,c]) = \frac{P_i(f,c)}{\sum_{j=1}^M P_j(f,c)}
\end{align}
\end{lemma}

\begin{proof}
For any given $\mathcal{F}_s$, $\{\mathcal{F}_d^i\}_{i=1}^M$ and $\mathcal{C}$, the optimum

\begin{equation} \nonumber
    \begin{aligned}
       \mathcal{D}^* &= \mathop{\arg\min}_{\mathcal{D}} \mathcal{J}_{\mathcal{D}} \\
       &= \mathop{\arg\min}_{\mathcal{D}}-\sum_{i=1}^M\mathbb{E}_{(f,c)\sim P_i}[\log\mathcal{D}_i([f, c])] \\
       &=\mathop{\arg\max}_{\mathcal{D}}\sum_{i=1}^M \int_{(f,c)} P_i(f,c)\log\mathcal{D}_i([f,c])d(f,c)\\
       &=\mathop{\arg\max}_{\mathcal{D}}\int_{(f,c)}\sum_{i=1}^M P_i(f,c)\log\mathcal{D}_i([f,c])d(f,c)
    \end{aligned}
\end{equation}

\noindent Here, we utilize the Lagrangian Multiplier for $\mathcal{D}^*$ under the condition (\ref{eq3}). We have:

\begin{align}
    L(\mathcal{D}_1, ..., \mathcal{D}_M, \lambda) = \sum_{i=1}^M P_i \log\mathcal{D}_i - \lambda(\sum_{i=1}^M\mathcal{D}_i-1)
\end{align}

Let $\nabla L = 0$, we have:
\begin{equation}\nonumber
    \left\{
             \begin{array}{lr}
             \nabla_{\mathcal{D}_i}\sum_{j=1}^M P_j \log\mathcal{D}_j - \lambda\nabla_{\mathcal{D}_i}(\sum_{j=1}^M\mathcal{D}_j-1) = 0 &  \\
             \sum_{i=1}^M\mathcal{D}_i = 1 & \\
             \end{array}
\right.
\end{equation}
From the two above equations, we have:

\begin{align}
    \mathcal{D}_i^*(f,c) = \frac{P_i(f,c)}{\sum_{j=1}^M P_j(f,c)}
\end{align}
\end{proof}

\begin{theorem}
\label{theo1}
Let $\widetilde{P}(f,c) = \frac{\sum_{i=1}^M P_i(f,c)}{M}$, when $\mathcal{D}$ is trained to its optimum $\mathcal{D}^*$, we have:
\begin{align}
    \mathcal{J}_{\mathcal{D}^*}=M \log M-\sum_{i=1}^M KL(P_i(f,c)||\widetilde{P}(f,c))
\end{align}
where $KL(\cdot)$ is the Kullback-Leibler (KL) divergence of each joint distribution $P_i(f,c)$ to the centroid $\widetilde{P}(f,c)$
\end{theorem}

\begin{proof}
Let $\widetilde{P}(f,c) = \frac{\sum_{i=1}^M P_i(f,c)}{M}$. We have:

\begin{equation} \nonumber
    \begin{aligned}
    \sum_{i=1}^M KL(P_i(f,c)||&\widetilde{P}(f,c)) = \\ &\sum_{i=1}^M\mathbb{E}_{(f,c)\sim P_i}[\log\frac{P_i(f,c)}{\widetilde{P}(f,c)}]
    \end{aligned}
\end{equation}

When $\mathcal{D}$ is updated to $\mathcal{D}^*$, we have:

\begin{equation} \nonumber
    \begin{aligned}
    \mathcal{J}_{\mathcal{D}^*} &= -\sum_{i=1}^M \mathbb{E}_{(f,c)\sim P_i}[\log\mathcal{D}^*_i([f,c])] \\
    &= -\sum_{i=1}^M \mathbb{E}_{(f,c)\sim P_i}[\log\frac{P_i(f,c)}{\sum_{j=1}^M P_j(f,c)}] \\
    &= -\sum_{i=1}^M \mathbb{E}_{(f,c)\sim P_i}[\log\frac{P_i(f,c)}{\sum_{j=1}^M P_j(f,c)} + \log M] \\
    &+ M \log M \\
    &= M \log M - \sum_{i=1}^M \mathbb{E}_{(f,c)\sim P_i}[\log\frac{P_i(f,c)}{\frac{\sum_{j=1}^M P_j(f,c)}{M}}] \\
    &= M \log M - \sum_{i=1}^M\mathbb{E}_{(f,c)\sim P_i}[\log\frac{P_i(f,c)}{\widetilde{P}(f,c)}] \\
    &= M \log M - \sum_{i=1}^M KL(P_i(f,c)||\widetilde{P}(f,c)) \\
    \end{aligned}
\end{equation}
\end{proof}

\noindent In our model, a mini-max game is implemented to achieve the optimum:

\begin{align}
    \min_{\mathcal{F}_s,\{\mathcal{F}_d^i\}_{i=1}^M, C}\max_{\mathcal{D}} \quad \mathcal{J}_{\mathcal{C}} + \lambda \mathcal{J}_{\mathcal{D}}
\end{align}

\noindent Therefore, by the non-negativity and convexity of the KL-divergence, we can have the corollary: \\

\begin{corollary}
\label{coro1}
When $\mathcal{D}$ is trained to its optimum $\mathcal{D}^*$, $\mathcal{J}_{\mathcal{D}^*}$ is $M \log M$. The optimum can be obtained if and only if $P_1(f,c) = P_2(f,c) = ... = P_M(f,c) = \widetilde{P}(f,c)$.
\end{corollary}

\begin{figure}
    \centering
    \includegraphics[width=0.8\columnwidth]{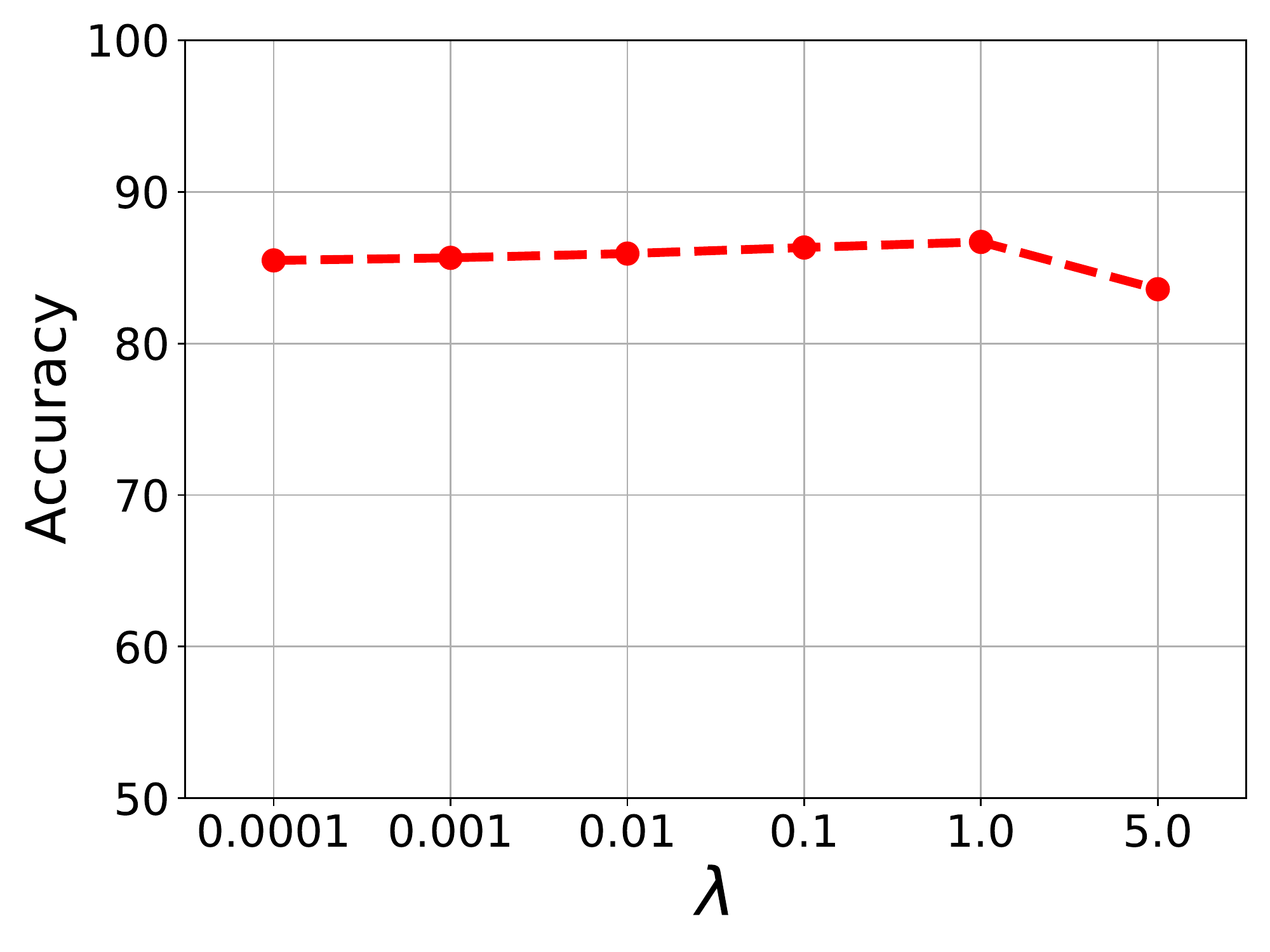}
    \caption{The parameter sensitivity analysis.
	}
    \label{Fig10}
\end{figure}

\subsection{Parameter Sensitivity Analysis}

The proposed method has one hyperparameter $\lambda$, which is used to balance $\mathcal{J}_\mathcal{C}$ and $\mathcal{J}_\mathcal{D}^E$. We conduct parameter sensitivity analysis on the Amazon review dataset. The $\lambda$ is evaluated in the range $\{0.0001,0.001,0.01,0.1,1.0,5.0\}$. The experimental results are shown in Figure \ref{Fig10}. The average classification accuracies across the four domains are reported. It can be noted that from 0.0001 to 1.0, the performance increases with $\lambda$ increasing, the performance change is very small. Then the accuracy reaches the optimum at the point $\lambda=1.0$, while the further increase of $\lambda$ will dramatically deteriorate the performance. This suggests that the selection of $\lambda$ has an influence on the system performance.

\end{document}